\documentclass{article}

\usepackage{multirow}
\usepackage{graphicx}
\usepackage{amsmath}
\usepackage{amsthm}
\usepackage{amssymb}
\newtheorem{theorem}{Theorem}
\usepackage{adjustbox}
\usepackage{subcaption}
\usepackage{booktabs}
\usepackage{float}
\usepackage[preprint]{neurips_2026}

\usepackage[utf8]{inputenc} 
\usepackage[T1]{fontenc}    
\usepackage{hyperref}       
\usepackage{url}            
\usepackage{booktabs}       
\usepackage{amsfonts}       
\usepackage{nicefrac}       
\usepackage{microtype}      
\usepackage{xcolor}         

\title{
Commutator-Induced Uncertainty in 
VAEs}

\author{%
  Tahereh Dehdarirad$^{1}$,
  Michael Felsberg$^{1}$,
  Gabriel Eilertsen$^{2}$,
  Ziliang Xiong$^{1}$ \\
  $^{1}$Computer Vision and Learning Systems (CVL), Linköping University, Sweden \\
  $^{2}$Department of Science and Technology, Linköping University, Sweden \\
  \\
  \texttt{\{tahereh.dehdarirad, michael.felsberg, zilian.xiong\}@liu.se} \\
  \texttt{gabriel.eilertsen@liu.se}
}

\begin{document}

\maketitle

\begin{abstract}
Variational autoencoders (VAEs) 
struggle to reflect non-commutativity.
Symmetry-aware VAEs 
address this by
enforcing commutativity 
through 
algebraic
regularization. While appropriate for commutative transformation groups, these
constraints impose a 
prior that suppresses meaningful non-commutative
structure 
even if intrinsic to the data. We argue that
non-commutativity should not be eliminated, but explicitly diagnosed and
reflected in reconstruction behaviour. We introduce a novel 
approach to Lie Group
VAEs that unifies geometric and algebraic perspectives on uncertainty
quantification, while isolating discrete generative factors from continuous geometric transformations. In Phase~1, the model is trained without structural constraints
while we empirically measure algebraic non-commutativity via finite Baker--Campbell--Hausdorff deviations and decoder order sensitivity via reconstruction order-swap tests. These diagnostics reveal a systematic scale
mismatch between latent non-commutativity and reconstruction behaviour under
unconstrained training. In Phase~2, we introduce a deformation--stability
constraint with a data-driven calibration constant that aligns decoder
sensitivity with algebraic non-commutativity. 
We evaluate the framework on dSprites, 3DShapes, 3DCars, and CelebA and compare it against both generic and symmetry-aware baselines, including $\beta$-VAE, CLG-VAE, and CFASL. Across the synthetic benchmarks, our method improves reconstruction quality and yields decoder-level behaviour that is more consistent with latent non-commutative structure. Qualitative analyses show clearer order-dependent latent compositions and more semantically stable reconstructions. 
On CelebA, our model yields more faithful reconstructions and 
identity-preserving, factor-specific latent traversals than CFASL, while also exhibiting meaningful order-dependent interactions between learned latent directions.


\end{abstract}

\vspace{-1em}
\section{Introduction}
Modern generative models rely on learned latent representations to capture the
factors of variation underlying high-dimensional data. In contemporary vision,
this paradigm extends well beyond classical latent-variable models: foundation
models and state-of-the-art generative systems, including latent diffusion
pipelines, derive much of their success from representations that compress,
structure, and disentangle semantic and geometric information. Variational
Autoencoders (VAEs) remain a central formal framework for studying such latent
spaces, since they combine probabilistic inference with an explicit generative
decoder and therefore provide a natural setting in which to analyse both
representation geometry and uncertainty.
Uncertainty quantification is essential for trustworthy machine learning,
particularly for generative models. VAEs 
map complex input data into a
lower-dimensional latent space while regularizing these latents to follow
a prior distribution. In their basic formulation, latent variables are
assumed to be independent or only linearly correlated, yet natural images
display strong nonlinear dependencies between generative factors
\cite{DBLP:journals/corr/abs-2404-15390}. While VAEs provide approximate
uncertainty estimates through their probabilistic framework, these
estimates are often miscalibrated \cite{10.1007/978-3-030-86520-7_6} and may
fail to reflect either the true geometry of the latent space
\cite{chadebec2020geometry} or the informativeness of the input data
\cite{10.1007/978-3-030-86520-7_6}. While appropriate for Abelian (commutative) transformation  groups, such assumptions exclude or suppress non-commutative interactions that arise naturally in many generative processes. This limitation is fundamental. Independent generative factors may interact in a non-commutative manner, and enforcing commutativity imposes a strong algebraic prior that can distort the learned transformation structure.  In our setting, non-commutativity models interactions between continuous generative factors whose joint effect depends on application order. This is relevant whenever latent transformations cannot be faithfully decomposed into independent commuting directions, so that composing two symmetry actions produces a residual effect not captured by a factorized latent model. Moreover, geometric uncertainty formulations based on symmetric covariance or information matrices capture curvature and sensitivity but cannot represent order-dependent effects. As a consequence, they systematically underestimate uncertainty in settings governed by non-commutative transformations. This oversight has direct implications for model reliability. Even properly calibrated systems with correctly estimated uncertainty are  affected by non-commutative transformation structure. In such settings, inherent uncertainty increases with commutator strength, causing standard calibration procedures to systematically underestimate reconstruction uncertainty. As we demonstrate empirically,
in the non-commutative case, order-dependent reconstruction effects grow faster than can be explained by the accumulated effects of partial actions, leading standard, factorized models to under-represent the resulting reconstruction ambiguity.  A second issue is the treatment of discrete generative factors. Lie groups model
continuous transformations, so applying them directly to categorical variation
can conflate discrete changes with continuous geometry. To address this, we introduce an explicit discrete latent variable, inferred
via a Gumbel--Softmax relaxation \cite{10.5555/3326943.3327009}, and keep it
decoupled from the Lie algebra. By conditioning the decoder on this variable
and holding it fixed during diagnostics, measured order effects reflect
continuous transformations rather than categorical variation. In this work, we argue that non-commutativity should not be eliminated, but explicitly diagnosed and reflected in uncertainty estimates. We introduce a diagnostic-driven framework for Lie Group VAEs that separates intrinsic algebraic non-commutativity from decoder-induced order effects at the level of reconstruction. Building on finite Baker--Campbell--Hausdorff (BCH) deviations and order-swap reconstruction tests, we reveal a systematic scale mismatch between latent non-commutativity and reconstruction behaviour under unconstrained
training. To address this mismatch, we extend the geometric perspective on VAEs to general Lie groups by introducing a deformation--stability principle.
This principle establishes a lower bound on reconstruction uncertainty in terms of the commutator norm of Lie algebra generators, analogous in spirit to uncertainty relations arising from non-commuting operators. Here, uncertainty refers to commutator-induced reconstruction sensitivity rather than posterior predictive variance.



This paper makes the following contributions: 
We introduce a {\bf novel diagnostic-driven framework} for Lie Group VAEs that empirically separates intrinsic algebraic non-commutativity from decoder-induced order effects, while isolating discrete generative factors from continuous geometric transformations. 
%
We extend {\bf geometric uncertainty modelling in VAEs} to general Lie
    groups via a deformation--stability principle, establishing a
    calibrated lower bound on reconstruction uncertainty induced by
    non-commuting latent generators. 
%
We use a {\bf pull-back metric on the Lie algebra} to quantify local decoder
sensitivity and combine it with finite non-commutativity diagnostics through
deformation--stability, yielding uncertainty estimates that jointly reflect
algebraic structure and geometric distortion.

\vspace{-1em}
\section{Related work}
\vspace{-1em}
\paragraph{Geometric perspectives}
Recent work has developed a \emph{geometric perspective} on VAEs. 
This perspective endows the latent space with a
Riemannian structure induced by the decoder or variational posterior.
However, Riemannian metrics are symmetric positive-definite bilinear forms by definition; they capture curvature and sensitivity but not asymmetric order-dependent effects.
\vspace{-1em}
\paragraph{Symmetry-aware representations learning}
A parallel line of work explores \emph{symmetry-aware latent representations}, in which
transformations are modelled explicitly through group structure \cite{Chau2020DisentanglingIW, 10.5555/3495724.3496839, Jung2024CFASLCF}.  These approaches have substantially enriched representation learning, but
typically rely, explicitly or implicitly, on commutativity or second-order
independence assumptions \cite{Zhu2021CommutativeLG}, or capture
non-commutative group actions only in restricted or partial settings
\cite{10.5555/3495724.3497379}. 
\vspace{-1em}
\paragraph{Uncertainty quantification in VAEs.}
Prior work has studied \emph{uncertainty in VAEs} from both encoder- and decoder-centric
perspectives. Encoder-based approaches analyze sensitivity of latent
representations via pullback metrics or variance structure
\cite{khan2023adversarial, zheng2022learning}, while decoder-based formulations
derive uncertainty from the Jacobian of the generative map or Fisher information
\cite{Arvanitidis2017LatentSO, Zacherl2021ProbabilisticAU, chadebec2020geometry}.
These methods capture curvature, anisotropy, and epistemic effects, but rely on
symmetric covariance or information matrices and therefore ignore 
order-dependent or non-commutative effects.

\vspace{-1.0em}
\section{Methodology}
\vspace{-1.2em}
\label{sec:methodology}

We adapt the methodology of \cite{Zhu2021CommutativeLG} and extend it to a
non-commutative setting with an additional discrete factor. The latent space is
factorized into continuous Lie-group variables, which govern geometric
transformations, and discrete variables, such as shape, whose categories define
learned canonical representations transformed before decoding. For each
generator pair \((A_i,A_j)\), we measure algebraic non-commutativity using a
finite BCH deviation \(D_{ij}\) and decoder-level order sensitivity using an
order-swap reconstruction difference \(\Delta_{ij}\). All diagnostics keep the discrete code fixed, so \(\Delta_{ij}\) reflects
continuous order effects rather than discrete variation. Training follows a fixed two-phase curriculum. Phase~1 trains without the
deformation--stability constraint while accumulating \(D_{ij}\) and
\(\Delta_{ij}\). Phase~2 fine-tunes the same model with the constraint
active, using a scale initialized from a fixed percentile of the empirical
ratios \(\Delta_{ij}/(D_{ij}+\epsilon)\). Figure~\ref{fig:nclgvae} illustrates the full architecture, including the discrete--continuous factorization and the auxiliary shape-prediction pathway.

\begin{figure}[t]
  \centering
  \includegraphics[width=\linewidth]{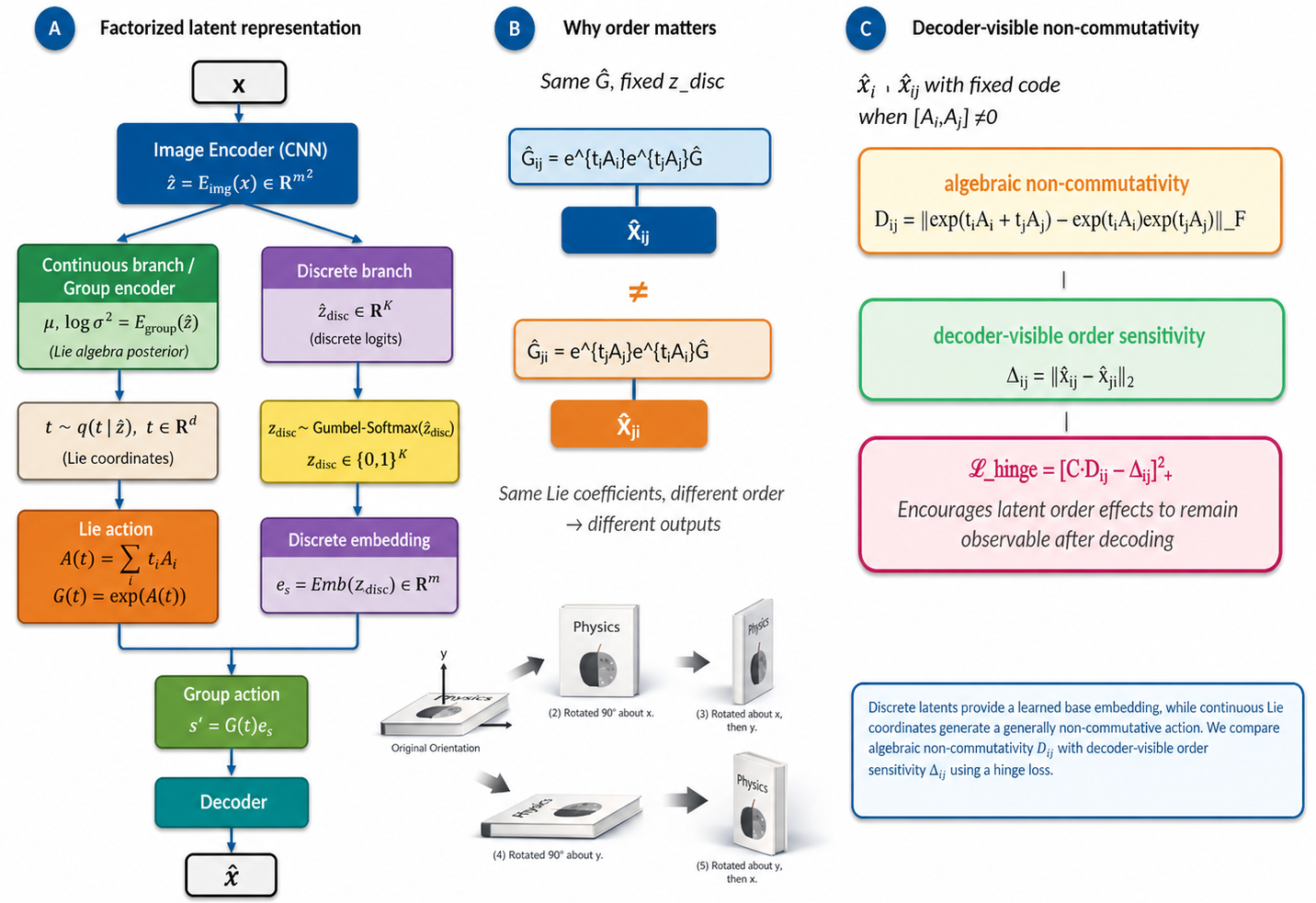}
  \caption{
  Commutator-aware VAE with discrete--continuous factorization. Discrete latents provide a learned base embedding, while continuous Lie coordinates induce transformations before decoding.
  Diagnostics compare algebraic and decoder-visible non-commutativity with the discrete code fixed; ordered products are used only in the pairwise order-swap diagnostics.
  }
  \label{fig:nclgvae}
  \vspace{-0.8em}
\end{figure}

\vspace{-1em}
\subsection{Unconstrained training and analysis: Phase 1}
\label{sec:phase 1}
We train a Lie Group VAE (CLG-VAE) \cite{Zhu2021CommutativeLG} without Hessian or
commutator penalties. Given an input image $x \in \mathcal{X}$, an image encoder
$E_{\mathrm{img}}$ produces representation
$\hat z = E_{\mathrm{img}}(x)$. This representation is mapped by a group
encoder $E_{\mathrm{group}}$ to a variational posterior
$(\mu, \log \sigma^2)$ over continuous Lie coordinates. We sample
$t = \mu + \sigma \odot \varepsilon$ with $\varepsilon \sim \mathcal{N}(0,I)$,
construct the Lie-algebra element $A(t)=\sum_j t_jA_j$, and form the group
element $G(t)=\exp(A(t))$, with $z=\mathrm{vec}(G(t))$ denoting its vectorized
group representation. In parallel, a separate discrete encoder produces logits
$\hat z_{\mathrm{disc}} \in \mathbb{R}^K$, from which a discrete latent variable
$z_{\mathrm{disc}} \in \{0,1\}^K$ is sampled using a Gumbel--Softmax relaxation. The image decoder $D_{\mathrm{img}}$ is conditioned on this sampled discrete
latent by embedding it as $e_s=\mathrm{Emb}(z_{\mathrm{disc}})$ and applying the
Lie-group action as $s'=G(t)e_s$ prior to image decoding. The core training objective is the unconstrained ELBO:


\vspace{-0.8em}
\[
\mathcal{L}_{\mathrm{VAE}}(x)
=
\ell_{\mathrm{recon}}(D_{\mathrm{img}}(s'),x)
+
\alpha \| z-\hat z \|_2^2
+
\beta \mathrm{KL}(q(t \mid \hat z)\|\mathcal{N}(0,I)).
\]
Here, \(\ell_{\mathrm{recon}}\) is the image reconstruction loss, \(\|z-\hat z\|_2^2\) enforces feature-space consistency rather than an identity constraint on \(G(t)\), and \(\beta\) scales the KL regularizer.
\vspace{-0.5em}

\paragraph{Discrete latent regularization.}
For models with a discrete latent variable, we add two auxiliary losses on the
discrete pathway: an InfoGAN-style prediction loss that makes the discrete code
recoverable from reconstructions, and a batch-usage regularizer that discourages
collapse to a single category. These losses act only on the discrete variables
and do not affect the Lie-group diagnostics or deformation--stability loss.
Details are given in Appendix~\ref{app:discrete_losses}.

\vspace{-0.5em}


\paragraph{Two-level non-commutativity diagnosis: discrepancy analysis.}
To distinguish intrinsic algebraic non-commutativity from decoder-induced order
effects, we introduce a two-level diagnostic framework. The foundation lies in
the BCH formula~\cite{Higham2008-ch10}, which
characterizes the deviation between \(\exp(A_i+A_j)\) and
\(\exp(A_i)\exp(A_j)\). When \([A_i,A_j]=0\), the generators commute and the two
expressions coincide; otherwise, higher-order BCH terms appear, with leading
contribution governed by the commutator. Rather than measuring
\(\|[A_i,A_j]\|_F\) directly, we use a finite-sample group-level deviation as an
operational proxy. Recent work by Godavarti~\cite{godavarti2025directionalnoncommutativemonoidalstructures}
formalizes order sensitivity \((A_iA_j\neq A_jA_i)\) as a structural property
of deep architectures. Motivated by this operational view, we define two
complementary tests for each generator pair \((A_i,A_j)\) to distinguish
group-induced from decoder-induced order effects:


\textbf{(1) Algebraic test (BCH deviation).}  
We compute
\begin{equation}
D_{ij}
=
\bigl\|
\exp(t_i A_i + t_j A_j)
-
\exp(t_i A_i)\exp(t_j A_j)
\bigr\|_F ,
\end{equation}
which measures finite Lie-group non-commutativity induced by the generators. Small $D_{ij}$ indicates near-commutativity at the group level. In practice, \(D_{ij}\) is evaluated per sample using \(t\sim q(t\mid \hat z)\) and then averaged over the diagnostic batch.

\textbf{(2) Decoder swap test.}
Here, we measure the decoder's response to permuted transformation order while holding the discrete code fixed.
Let \(G_k=\exp(t_kA_k)\), \(\hat G=\mathrm{mat}(\hat z)\), and
\(e_s=\mathrm{Emb}(z_{\mathrm{disc}})\), with \(\hat z=E_{\mathrm{img}}(x)\). Then
\begin{equation}
\small
\Delta_{ij}
=
\mathbb{E}_{x,t,z_{\mathrm{disc}}}
\!\left\|
D_{\mathrm{img}}(G_iG_j\hat G e_s)-
D_{\mathrm{img}}(G_jG_i\hat G e_s)
\right\|_2 ,
\end{equation}
where \(x\sim p_{\mathrm{data}}\), \(t\sim q(t\mid \hat z)\), and
\(z_{\mathrm{disc}}\sim q_{\mathrm{disc}}(\cdot\mid x)\).
Here, \(\Delta_{ij}\approx 0\) indicates decoder order-invariance, while larger
values indicate decoder sensitivity. Jointly, \(D_{ij}\) and \(\Delta_{ij}\)
categorize pairs as commutative \((D_{ij}\approx0,\Delta_{ij}\approx0)\),
algebraic-only non-commutative \((D_{ij}>0,\Delta_{ij}\approx0)\), or strongly
non-commutative \((D_{ij}>0,\Delta_{ij}>0)\).

\vspace{-0.5em}

\subsection{Manifold-based uncertainty regularizer with deformation stability: Phase~2}
\label{subsec:uncertainty_regularizer}

The core principle guiding Phase~2 is \emph{deformation--stability}: 
geometric sensitivity in reconstruction space should scale consistently
with algebraic non-commutativity in latent space.
Intuitively, when latent transformations fail to commute, their joint
action induces intrinsic deformations of the latent manifold, which the
decoder should express as increased local uncertainty rather than suppress. The link between commutator norms and deformation is mediated by the BCH
expansion. For two generators \(A_i\) and \(A_j\), small latent transformations
satisfy \(\exp(t_iA_i)\exp(t_jA_j)=\exp\!\left(t_iA_i+t_jA_j+\tfrac12 t_it_j[A_i,A_j]+O(\|t\|^3)\right)\).Thus, the commutator governs the leading non-additive correction to
independent composition. When \([A_i,A_j]\neq 0\), this correction introduces
an intrinsic mixed term, making the latent displacement order-dependent.
After decoding, this appears as a difference between reconstructions under
permuted transformation order. Hence, \(\|[A_i,A_j]\|_F\) provides an
algebraic proxy for intrinsic non-commutative deformation. We operationalize this principle by linking reconstruction sensitivity to
generator non-commutativity through a deformation--stability framework
inspired by \cite{bronstein2021geometric}. Specifically, we measure how small
perturbations along latent generator directions propagate through the
Lie-group action and decoder, yielding an empirical notion of
\emph{manifold-aware uncertainty}. This uncertainty is regularized to remain
consistent with the algebraic structure identified in Phase~1.

\vspace{-0.5em}

\paragraph{Theoretical foundation: From torsion to geometric uncertainty.}
We motivate our Phase~2 regularization through an analogy between algebraic non-commutativity and geometric deformation, drawing on two complementary theoretical perspectives.\\
%
\textbf{Non-commutativity as torsion (analogy).}
    In differential geometry, non-commutativity of vector fields is formalized through the torsion tensor of a connection, defined by the Lie bracket $T(X,Y) = [X,Y]$ \cite{doi:10.1137/19M1252879}. While our model does not instantiate a full post-Lie or connection-based geometric structure, this perspective provides a useful conceptual correspondence:
a) the learned generators $\{A_i\}$ correspond to elements of a Lie algebra $\mathfrak{g}$,
b) the decoder mapping induces a geometry on latent space via its local Jacobian, and
c) generator commutators $[A_i, A_j]$ act as torsion-like quantities, with $\|[A_i,A_j]\|_F$ capturing the intrinsic non-commutativity of the generator pair. As with rotations about different axes in \(\mathbb{R}^3\), two latent symmetry
directions need not commute: changing their order can induce a residual
displacement that appears after decoding as a difference between reconstructed
outputs.

%
\textbf{Deformation–stability principle.}
    Following the deformation-stability framework of \cite{bronstein2021geometric}, we require that changes in the output of a representation be controlled by the magnitude of the underlying transformation. Formally, for a transformation $\tau$ acting via a representation $\rho$, deformation stability requires
    \[
    \|f(\rho(\tau) z) - f(z)\| \leq C \, c(\tau) \, \|z\|,
    \]
    where $c(\tau)$ measures the size of the deformation and $C>0$ is a global stability constant. In our Lie-group setting, non-commuting generators $A_i$ and $A_j$ induce deformations of the group action that cannot be reduced to independent effects. We therefore associate the deformation magnitude $c(\tau)$ with a torsion-like quantity proportional to $\|[A_i,A_j]\|_F$, and measure reconstruction sensitivity using a joint, worst-case decoder response $U^{\mathrm{manifold}}_{ij}$. This replaces factorized uncertainty measures (e.g.\ $\sigma_i \sigma_j$), which implicitly assume independence and systematically underestimate uncertainty for entangled generator pairs.
\begin{theorem}[Geometric interpretation of reconstruction sensitivity]
Let $x_{\mathrm{rec}}=D_{\mathrm{img}}(s'(t))$ with
$s'(t)=G(t)e_s$, where $G(t)=\exp(A(t))$ and the discrete embedding $e_s$ is held fixed.
The joint sensitivity measure
\begin{equation}
U_{ij}^{\mathrm{manifold}}
=
\sigma_{\max}\!\left(
\left[
\frac{\partial x_{\mathrm{rec}}}{\partial t_i},
\frac{\partial x_{\mathrm{rec}}}{\partial t_j}
\right]
\right)
\end{equation}

quantifies the worst-case sensitivity of reconstructions to changes in  generator parameters $(t_i, t_j)$, measured through the Riemannian geometry induced by the decoder.
\end{theorem}

\noindent
A proof is provided in Appendix~\ref{app:theorom 1}.

\textbf{Stability postulate (empirical form).}
Following the analogy between commutators and torsion~\cite{doi:10.1137/19M1252879} and the deformation--stability principle~\cite{bronstein2021geometric}, we postulate that geometric sensitivity in reconstruction space should not be
suppressed below a calibrated lower bound implied by the degree of algebraic non-commutativity present in latent space. Rather than enforcing deformation--stability using infinitesimal commutator norms and Jacobian-based sensitivities, we adopt finite-action empirical proxies that are directly computable during training. For each generator pair $(i,j)$, we evaluate the previously defined
group-level non-commutativity proxy $D_{ij}$ and the decoder-level
order-sensitivity proxy $\Delta_{ij}$, which together characterize the
relationship between latent algebraic structure and reconstruction behaviour. Using these proxies, we enforce the empirical deformation--stability condition:
\begin{equation}
\Delta_{ij} \;\ge\; C \cdot D_{ij},
\label{eq:empirical_deformation_stability}
\end{equation}
where $C>0$ is a global stability constant. This inequality is not a derived result but a modelling assumption, enforced during Phase~2 training to ensure consistency between latent group structure and reconstruction behaviour.
This principle enforces a lower bound on reconstruction sensitivity induced by the level of non-commutativity of generators. The analogy is reminiscent of the Heisenberg uncertainty principle, where non-commuting operators impose irreducible uncertainty, but here it serves as a modelling constraint rather than a physical law. 
\vspace{-0.8em}
\paragraph{Deformation hinge loss}
\label{subsec:uncertainty_regularizer}
We enforce the deformation--stability principle using a one-sided hinge 
that penalizes cases where the decoder fails to reflect latent non-commutativity.
For each generator pair $(i,j)$, the loss is defined as
\begin{equation}
\label{eq:pairwise-hinge-loss}
\ell_{ij}
=
\big[\max\big(0,\; C\,\bar D_{ij} - \bar \Delta_{ij}\big)\big]^2 .
\end{equation}
This penalty is active only when the observed decoder order sensitivity
$\bar\Delta_{ij}$ is smaller than the level implied by the algebraic
non-commutativity $\bar D_{ij}$.
In this regime, the decoder behaves as if the transformations commute,
despite non-commutativity in the latent generators. The hinge therefore enforces a lower bound on reconstruction sensitivity without constraining or amplifying decoder responses beyond this minimum.
A central component of our method is a scalar constant $C>0$ that relates
\emph{algebraic non-commutativity} in latent space to its observable effect
in the decoder. The quantities \(D_{ij}\) and \(\Delta_{ij}\) need not be comparable in absolute scale, since \(D_{ij}\) is measured in latent/group space whereas \(\Delta_{ij}\) is measured in reconstruction space. The constant \(C\) therefore calibrates this cross-space scale mismatch, absorbing the typical transfer from latent non-commutative deviation to observable reconstruction discrepancy. 
Rather than fixing \(C\) as a hyperparameter, we estimate it empirically from data and refine it during training; details are provided in Appendix~\ref{app:Empirical calibration of C}.

\vspace{-1em}
\section{Experiments}
\label{subsec:exp_setup}
\paragraph{Datasets.}
\vspace{-1em}
We evaluate our method on the widely used synthetic datasets \textbf{dSprites} \cite{higgins2017betavae},
\textbf{3DShapes} \cite{3dshapes18}, and \textbf{3DCars} \cite{reed2015deep}. These datasets provide controlled variation over
generative factors, including both continuous geometric attributes and
discrete object-level factors. This makes them well suited to our setting,
since they support the study of structured latent transformations while making
it easier to analyse order-dependent interactions between continuous latent
transformations and to limit confounding from categorical variation.
Moreover, these datasets are standard benchmarks in prior work on
symmetry-aware and structured latent representation learning, including the
baselines we compare against. This enables direct comparison under a shared
experimental protocol. We additionally evaluate on \textbf{CelebA} \cite{liu2015deep} as a larger,
non-synthetic dataset to assess whether the proposed framework remains
effective beyond synthetic benchmarks and generalizes to more realistic visual
data. Further dataset details and preprocessing procedures are provided in Appendix \ref{app:dataset_details}

\vspace{-1em}
\paragraph{Model variants and baselines.}
We evaluate the proposed framework through three internal variants of our
model, in addition to three external baselines. The first internal variant,
reported as \textbf{LG-VAE (mixed variables)} in
Table~\ref{tab:quantitative_results}, uses a joint latent representation for both
continuous and discrete factors and is trained without commutative
constraints.  This is the CLG VAE architecture, which is trained without Hessian or commutator penalties as described in Section~\ref{sec:phase 1}; The second internal variant, reported as \textbf{Ours (P1 only, P1 epochs)} uses the proposed separate continuous--discrete latent factorization and is trained only for the Phase~1 duration, without
commutativity or deformation--stability regularization. \textbf{Ours (P1 only, P1+P2 epochs)} is a duration-matched Phase~1 control: it uses the same architecture and  factors as the full model and is trained for the same total number of epochs as P1+P2, but never enables the Phase~2 deformation--stability constraint. \textbf{Ours (P1+P2)} starts from the same \textbf{Ours (P1 only, P1 epochs)} setting and then continues
training with the Phase~2 deformation--stability hinge enabled. This is the full
proposed method.  As external baselines, we consider \textbf{\(\beta\)-VAE} \cite{higgins2017betavae}, \textbf{CLG-VAE} \cite{Zhu2021CommutativeLG},
and \textbf{CFASL} \cite{Jung2024CFASLCF}. \(\beta\)-VAE provides a standard VAE baseline with
factorized latent regularization but without explicit symmetry structure,
allowing us to compare against a widely used model that captures
global latent organization while not modelling group actions or
non-commutative interactions explicitly. By contrast, CLG-VAE and CFASL are
symmetry-aware baselines that impose commutative structure on latent
generators. These latter baselines are relevant to our objective
because they instantiate the 
assumption that latent symmetry
generators should commute, whereas our framework is designed to retain and
measure non-commutative interactions. 
Detailed implementation, training hyperparameters, and evaluation settings are provided in Appendix \ref{app:exp_setting} 
\paragraph{Quantitative evaluation.}Our primary evaluation focuses on whether latent non-commutativity is reflected
in the reconstruction
. During training, we monitor the calibrated
deformation--stability ratio
\[
R_{ij} = \frac{\Delta_{ij}}{C D_{ij}+\epsilon},
\qquad
\bar R = \mathbb{E}_{i<j}[R_{ij}],
\]
which provides a dynamic diagnostic of whether decoder-level order sensitivity tracks latent algebraic non-commutativity across the transition from Phase~1 to Phase~2. Values below one indicate that the decoder under-reflects the non-commutativity measured in latent space, while values close to or above one indicate consistency with the calibrated deformation--stability lower bound. For the final model diagnostics, we report reconstruction loss and the FactorVAE metric (FVM) \cite{pmlr-v80-kim18b} , the latter serving as an auxiliary representation-quality measure to verify that the proposed constraint does not degrade standard factor recovery. The stability violation and calibrated empirical scale \(C_{\mathrm{emp}}\) are provided in the diagnostic plots.

\vspace{-1em}
\paragraph{CelebA evaluation}
On CelebA, we mainly focus on qualitative fidelity and latent-structure analysis, since the dataset lacks clean ground-truth generative factors, while also reporting FID as a complementary measure of generative quality. We assess reconstruction quality by comparing each image with its reconstruction, latent traversals by varying one latent dimension at a time, and order-dependent interactions by reversing pairs of latent transformations with different step magnitudes and comparing the outputs.


\vspace{-1em}
\section{Results}
\vspace{-1em}
\subsection{Synthetic Benchmarks}

\vspace{-1em}

\subsubsection{Quantitative analysis: Phase-aware constraint diagnostics.}

Figures~\ref{fig:cars3d_phase_diag} and \ref{fig:3dshapes_phase_diag}
summarize the training dynamics of the phase-aware constraint on Cars3D and
3DShapes; the corresponding dSprites diagnostic is provided in
Appendix~\ref{app:phase_diagnostics_additional}. The top panel compares decoder
deformation \(\Delta_{ij}\) with a consistently scaled latent
non-commutativity diagnostic, while the middle panel reports the
deformation--stability ratio \(\bar{R}\). Since the constraint imposes a
calibrated lower bound, the relevant regime is \(\bar{R}\geq 1\): Cars3D remains
clearly above this boundary during Phase~2, whereas 3DShapes stabilizes closer
to it, indicating tighter satisfaction of the same bound. The bottom panel shows
that reconstruction loss remains stable or decreases through Phase~2, suggesting
that calibrated algebraic consistency is enforced without degrading
reconstruction quality. Table~\ref{tab:quantitative_results} reports quantitative results averaged over
five random seeds. Ours (P1+P2) achieves the strongest FVM on dSprites and
3DCars while remaining competitive on 3DShapes, with reconstruction comparable
to the duration-matched P1-only control. This comparison shows that Phase~2
preserves reconstruction quality while improving FVM on dSprites and 3DCars.
Thus, the Phase~2 gains are not explained by training duration alone, but reflect
improved latent structure and calibrated non-commutative consistency.

\begin{figure}[t]
    \centering

    \begin{subfigure}[t]{0.49\columnwidth}
        \centering
        \phantomsubcaption\label{fig:cars3d_phase_diag}
        \textbf{\small (\subref{fig:cars3d_phase_diag}) Cars3D}\par
        \vspace{0.15em}
        \includegraphics[width=\linewidth,trim=8 8 8 8,clip]{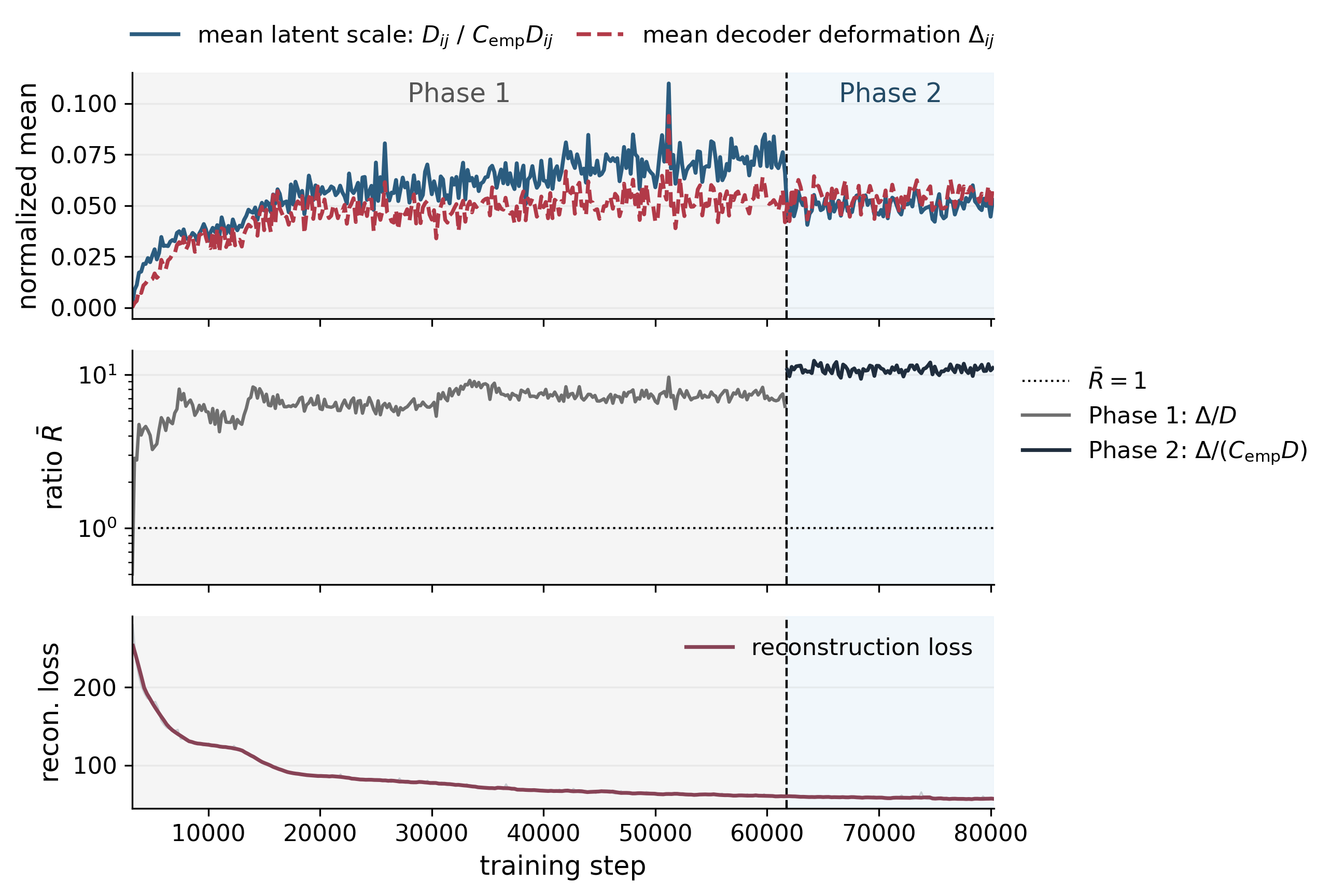}
    \end{subfigure}\hfill
    \begin{subfigure}[t]{0.49\columnwidth}
        \centering
        \phantomsubcaption\label{fig:3dshapes_phase_diag}
        \textbf{\small (\subref{fig:3dshapes_phase_diag}) 3DShapes}\par
        \vspace{0.15em}
        \includegraphics[width=\linewidth,trim=8 8 8 8,clip]{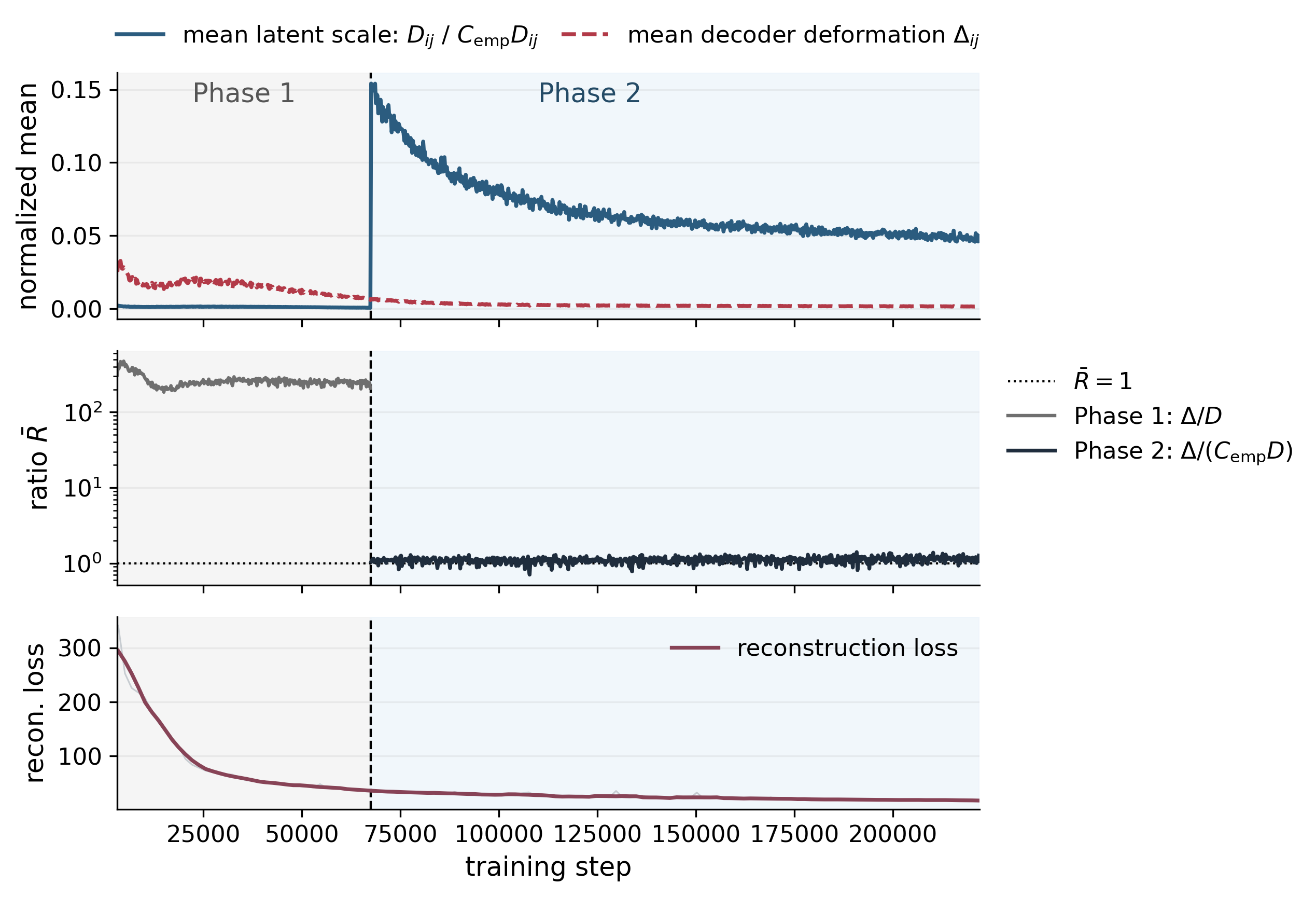}
    \end{subfigure}

    \caption{
    Phase-aware diagnostics on Cars3D and 3DShapes.
    Top: $\Delta_{ij}$ and a consistently scaled latent non-commutativity diagnostic across both phases.
    Middle: deformation--stability ratio $\bar{R}$ with reference boundary $\bar{R}=1$.
    Bottom: reconstruction loss.
    }
    \label{fig:phase_diag_main}
\end{figure}






\subsubsection{Qualitative analysis}

We complement Table~\ref{tab:quantitative_results} with qualitative examples of latent traversals, ordered compositions, and reconstructions, showing that the proposed model yields more coherent and semantically stable transformations across latent and image space.

\vspace{-0.3em}

\vspace{-1em}
\paragraph{Reconstruction quality for 3DShapes.}
We compare reconstructions from the original CLG-VAE and our model on the exact same input samples for 3DShapes in Figure \ref{fig:reconstruction_combined_a}. The first row shows the original images, the second and third rows show reconstructions from CLG-VAE and our model, respectively, and the last two rows show the corresponding absolute error maps. Darker red indicates larger reconstruction error. Our model reconstructs object regions with more faithful colour filling and reduced blur, while also preserving sharper contours and cleaner shape structure.
These qualitative improvements are reflected in uniformly smaller reconstruction errors relative to the original CLG-VAE on the shown examples.
\vspace{-1em}
\paragraph{Reconstruction quality across datasets and model variants.}

Figure~\ref{fig:reconstruction_combined_b} compares reconstructions on dSprites and 3DShapes for Variants~1 and~3. Across both datasets, Variant~3 better matches the ground truth, preserving object identity, shape structure, and fine-grained attributes, while Variant~1 produces blurrier reconstructions with weaker semantic detail. This gap is clearest on dSprites: Variant~1 captures continuous transformations such as translation but collapses discrete factors to a single shape category, whereas Variant~3 preserves both continuous Lie-group transformations and discrete generative factors. Overall, the full model improves reconstruction fidelity and captures the joint continuous--discrete structure more effectively.

\vspace{-1.2em}

\paragraph{Order-dependent composition of translation- and rotation-like latents on dSprites.}
Figure~\ref{fig:dsprites_ellipse_traversal_and_composition} shows the individual traversals used for composition analysis, where \(a\) is translation-like and \(b\) rotation-like, then compares the ordered compositions \(ab\) and \(ba\) on an ellipse sample across traversal strengths. Since transformations are applied in latent group coordinates rather than around a fixed pixel-space point, rotations act about an object-centred implicit anchor, while translations move that anchor across the image. Our model shows a clear order-dependent effect while preserving ellipse identity. This appears in the pixel-wise difference maps and, at larger steps, in centroid-and-axis overlays, where \(ab\) and \(ba\) produce different final centroid locations and orientations. CLG-VAE also shows non-zero order dependence, but its compositions are less semantically stable, with stronger deformation and greater deviation from the ellipse manifold. This is consistent with its objective, which uses only a Hessian penalty (\texttt{hy\_hes}=40, \texttt{hy\_commute}=0), so residual order dependence is not explicitly preserved. Overall, our model preserves non-commutative latent interactions more clearly while maintaining better semantic stability under composition.

\begin{figure*}[t]
    \centering
    \setlength{\tabcolsep}{8pt}
    \renewcommand{\arraystretch}{1.0}

    \begin{tabular}{cc}
        \subcaptionbox{\textbf{3DShapes: CLG-VAE vs Ours.}\label{fig:reconstruction_combined_a}}[0.34\textwidth]{
            \begin{minipage}[t]{0.34\textwidth}
                \vspace{-3.3cm}
                \centering
                \includegraphics[
                    width=\linewidth,
                    height=0.50\textheight,
                    keepaspectratio
                ]{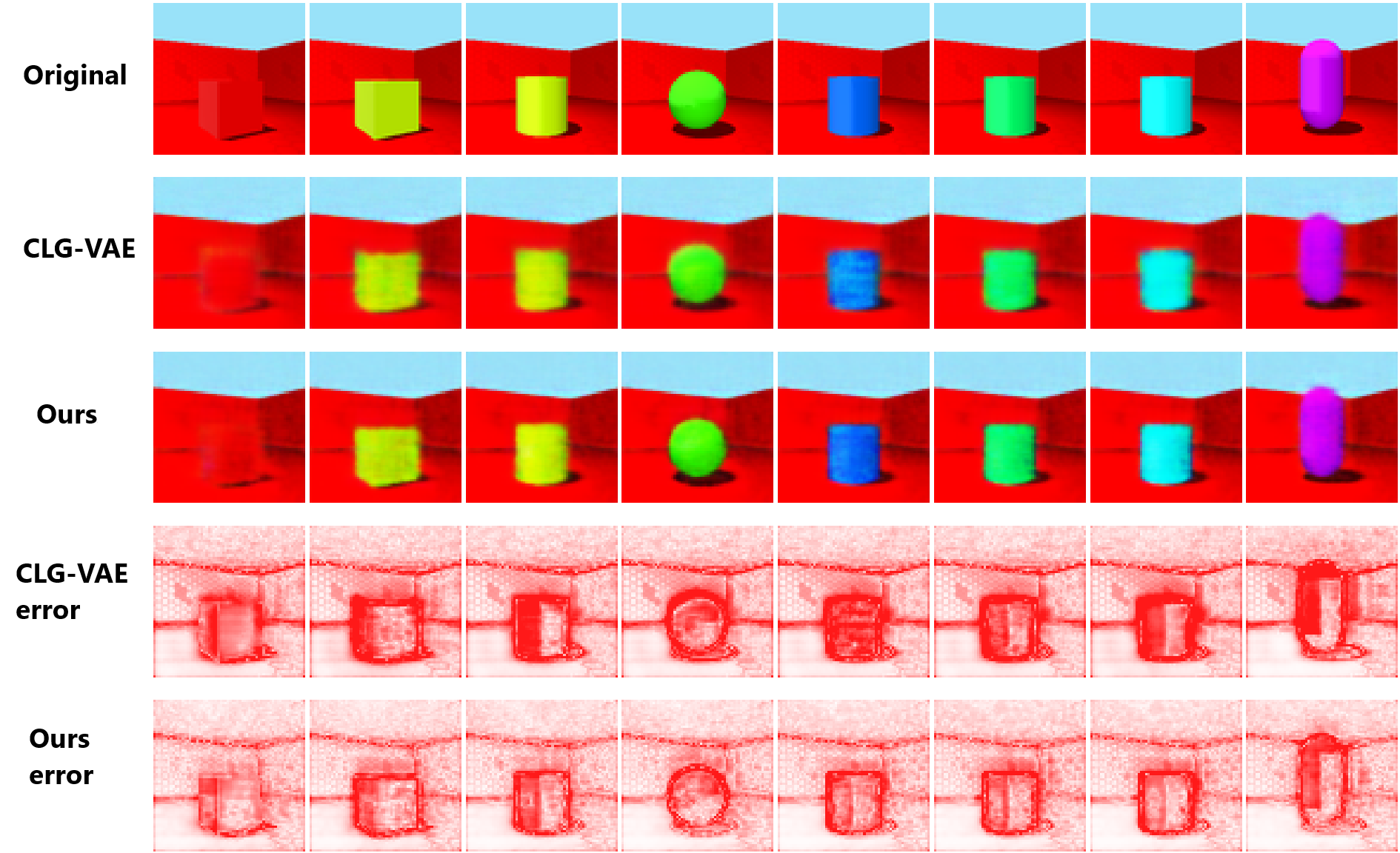}
            \end{minipage}
        }
        &
        \subcaptionbox{\textbf{Reconstruction quality across datasets and variants.}\label{fig:reconstruction_combined_b}}[0.62\textwidth]{
            \begin{minipage}[t]{0.62\textwidth}
                \vspace{0pt}
                \centering

                \begin{minipage}[t]{0.48\linewidth}
                    \centering
                    \includegraphics[
                        width=\linewidth,
                        height=0.22\textheight,
                        keepaspectratio
                    ]{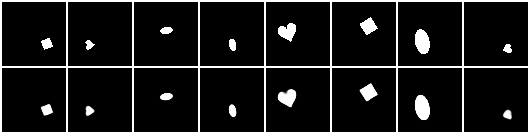}

                    \vspace{0.2em}
                    {\small \textbf{dSprites (Variant~3).}}
                \end{minipage}
                \hfill
                \begin{minipage}[t]{0.48\linewidth}
                    \centering
                    \includegraphics[
                        width=\linewidth,
                        height=0.22\textheight,
                        keepaspectratio
                    ]{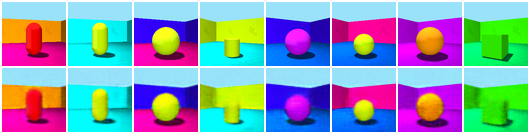}

                    \vspace{0.2em}
                    {\small \textbf{3DShapes (Variant~3).}}
                \end{minipage}

                \vspace{0.3em}

                \begin{minipage}[t]{0.48\linewidth}
                    \centering
                    \includegraphics[
                        width=\linewidth,
                        height=0.22\textheight,
                        keepaspectratio
                    ]{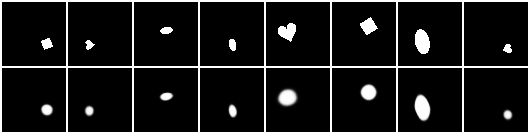}

                    \vspace{0.2em}
                    {\small \textbf{dSprites (Variant~1).}}
                \end{minipage}
                \hfill
                \begin{minipage}[t]{0.48\linewidth}
                    \centering
                    \includegraphics[
                        width=\linewidth,
                        height=0.22\textheight,
                        keepaspectratio
                    ]{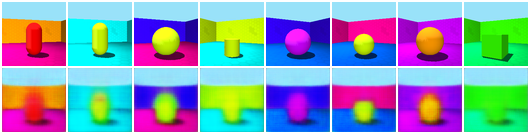}

                    \vspace{0.2em}
                    {\small \textbf{3DShapes (Variant~1).}}
                \end{minipage}
            \end{minipage}
        }
    \end{tabular}

    \caption{Qualitative reconstruction comparisons.
    (a) On 3DShapes, our model reconstructs more accurately than CLG-VAE, with lower error.
    (b) Across dSprites and 3DShapes, Variant~3 (Ours: P1+P2) yields more faithful reconstructions than Variant~1 (LG-VAE: mixed variables).}
    \label{fig:reconstruction_combined}
    \vspace{-1em}
\end{figure*}

\begin{figure*}[t]
    \centering

    \begin{minipage}[t]{0.49\textwidth}
        \vspace{1.0cm}
        \centering
        \includegraphics[
            width=\linewidth,
            keepaspectratio
        ]{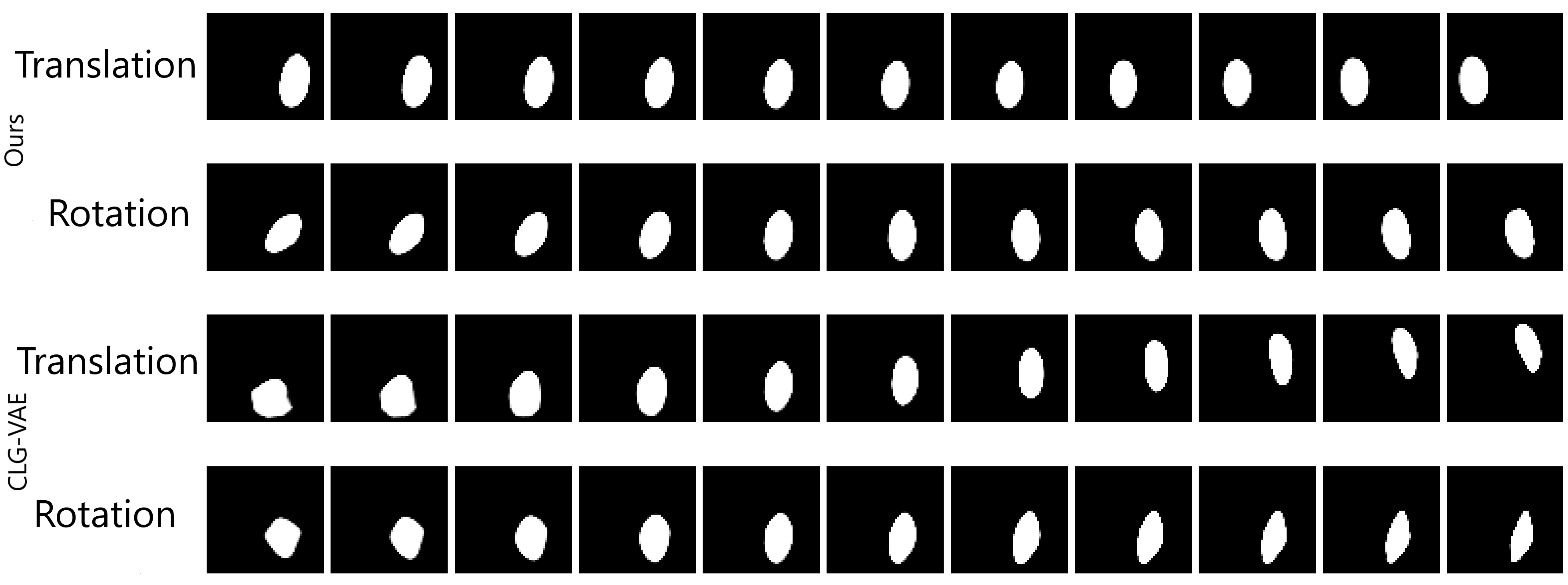}

        \vspace{0.2em}
        {\small\textbf{(A)} Single-latent traversals}
    \end{minipage}
    \hfill
    \begin{minipage}[t]{0.49\textwidth}
        \vspace{0pt}
        \centering
        \includegraphics[
            width=\linewidth,
            keepaspectratio
        ]{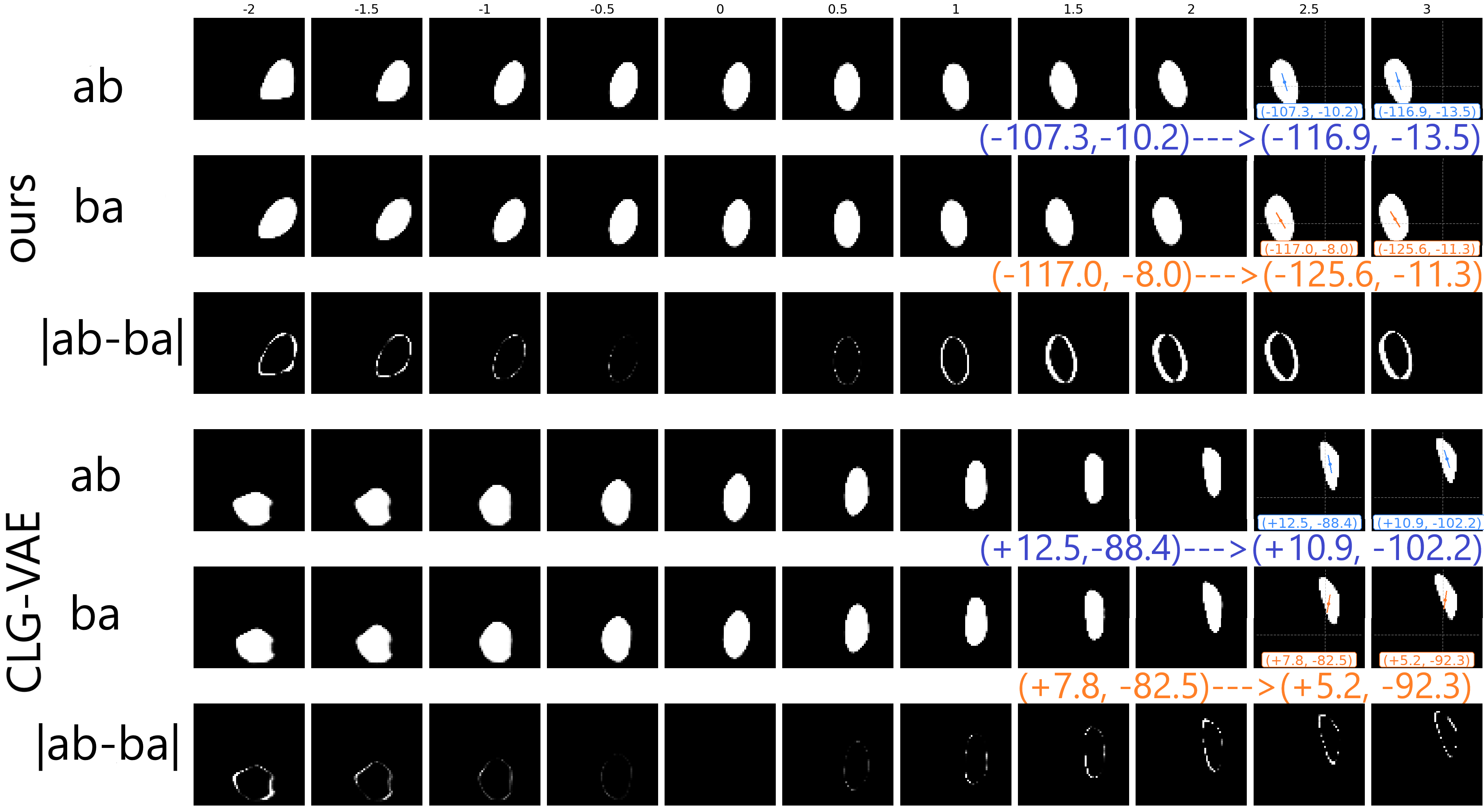}

        \vspace{0.2em}
        {\small\textbf{(B)} Ordered compositions}
    \end{minipage}

\caption{Single-latent traversals and order-dependent composition on dSprites.
\textbf{(A)} Individual latents used in the ellipse analysis, where \(a\) denotes a translation-like latent and \(b\) a rotation-like latent.
\textbf{(B)} Ordered compositions \(ab\), \(ba\), and the pixel-wise difference map \(|ab-ba|\) for the same ellipse sample, with strength swept from \(t=-2\) to \(t=3\). The final two columns show centroid-axis overlays.}

    \label{fig:dsprites_ellipse_traversal_and_composition}
    \vspace{-0.5em}
\end{figure*}

\vspace{-1em}
\subsection{CelebA}

\paragraph{Latent traversals.}
Figure~\ref{fig:shapeon-cfasl-final5} compares latent traversals from our model and CFASL. In each row, the first image is the input and the second is its reconstruction. Qualitatively, our model reconstructs the input more faithfully and yields more identity-preserving, factor-specific traversals. This is also reflected in generative quality: under a shared CelebA FID-50k protocol, our model attains a lower FID than CFASL (66.27 vs.\ 77.59). For example, our model better separates background colour from skin tone, preserves facial identity under hair colour/style changes, and isolates hair colour more cleanly, whereas CFASL exhibits stronger entanglement with face shape, skin tone, eyebrows, hair style, and rotation. We provide additional qualitative analysis of order-dependent latent interactions on CelebA in Appendix~\ref{app:celeba_further_qualitative}, including the example in Figure~\ref{fig:celeba_order_reversal}.

\begin{figure*}[h]
    \centering
    \includegraphics[width=\textwidth]{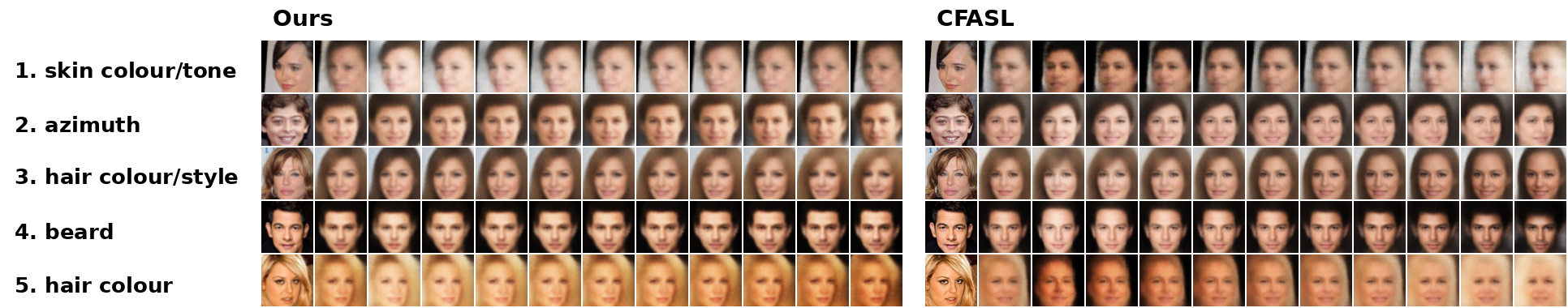}
    \caption{Latent traversal comparison between our model and CFASL.
    In each row, the first image is the input and the second is its reconstruction.}
    \label{fig:shapeon-cfasl-final5}
    \vspace{-1em}
\end{figure*}

\begin{table*}[t]
\centering
\footnotesize
\setlength{\tabcolsep}{5pt}
\renewcommand{\arraystretch}{1.0}
\caption{
Quantitative evaluation averaged over 5 random seeds.
We compare with \(\beta\)-VAE, CLG-VAE, and CFASL.
CLG-VAE is continuous-only, while all Ours variants model both continuous and
discrete factors.
``Ours (P1 only, P1+P2 epochs)'' is a duration-matched control trained for the
same number of epochs as Ours (P1+P2), but without Phase~2 or the
deformation--stability constraint.
}

\label{tab:quantitative_results}

\begin{tabular}{lcccccc}
\toprule
& \multicolumn{2}{c}{dSprites} & \multicolumn{2}{c}{3DShapes} & \multicolumn{2}{c}{3DCars} \\
\cmidrule(lr){2-3} \cmidrule(lr){4-5} \cmidrule(lr){6-7}
Model
& Recon.\(\downarrow\) & FVM\(\uparrow\)
& Recon.\(\downarrow\) & FVM\(\uparrow\)
& Recon.\(\downarrow\) & FVM\(\uparrow\) \\
\midrule
\(\beta\)-VAE
& 18.57 \(\pm\) 0.90 & 0.60 \(\pm\) 0.02
& 22.72 \(\pm\) 2.19 & 0.75 \(\pm\) 0.07
& 59.09 \(\pm\) 1.18 & 0.85 \(\pm\) 0.02 \\

CLG-VAE
& 18.95 \(\pm\) 1.91 & 0.86 \(\pm\) 0.02
& 27.06 \(\pm\) 2.03 & \textbf{0.93 \(\pm\) 0.05}
& 67.58 \(\pm\) 2.71 & 0.91 \(\pm\) 0.03 \\

CFASL
& 23.37 \(\pm\) 1.18 & 0.80 \(\pm\) 0.03
& 22.35 \(\pm\) 1.95 & 0.83 \(\pm\) 0.02
& 66.88 \(\pm\) 0.50 & 0.88 \(\pm\) 0.02 \\

LG-VAE & 61.88 \(\pm\) 0.61 & 0.61 \(\pm\) 0.07
& 67.91 \(\pm\) 9.26 & 0.68 \(\pm\) 0.04
& 87.10 \(\pm\) 1.88 & 0.88 \(\pm\) 0.06 \\

Ours (P1 only, P1 epochs)
& 24.64 \(\pm\) 2.96 & 0.83 \(\pm\) 0.03
& 30.72 \(\pm\) 3.61 & 0.82 \(\pm\) 0.03
& 67.96 \(\pm\) 4.64 & 0.92 \(\pm\) 0.07 \\

Ours (P1 only, P1+P2 epochs)
& \textbf{13.38 \(\pm\) 0.28} & 0.83 \(\pm\) 0.02
& \textbf{15.67 \(\pm\) 0.05} & 0.80 \(\pm\) 0.09
&\textbf{ 58.17 \(\pm\) 2.76} & 0.87 \(\pm\) 0.07 \\

Ours (P1+P2)
& 13.96 \(\pm\) 0.90 & \textbf{0.89 \(\pm\) 0.09}
& \textbf{15.67 \(\pm\) 1.87} & 0.84 \(\pm\) 0.04
& 58.44 \(\pm\) 2.87 & \textbf{0.95 \(\pm\) 0.04} \\

\bottomrule
\end{tabular}
\end{table*}

\section{Conclusions}
\vspace{-1em}
We introduced a diagnostic-driven framework for VAEs that treats
non-commutativity not as a nuisance to be regularized away, but as a measurable
source of uncertainty that should be diagnosed and reflected in reconstruction
behaviour. The proposed framework separates continuous geometric transformations
from discrete generative factors, measures latent algebraic
non-commutativity through finite BCH deviations, and relates it to
decoder-level order sensitivity through a calibrated deformation--stability
constraint. Empirically, our results show that unconstrained training leads to a systematic
mismatch between latent non-commutativity and reconstruction sensitivity. The
proposed Phase~2 regularization reduces this mismatch by aligning decoder
behaviour with the learned latent algebraic structure, while preserving or
improving reconstruction quality. Across the synthetic benchmarks, the method
yields stronger consistency between latent and reconstruction-level behaviour,
clearer order-dependent latent compositions, and more semantically stable
reconstructions under composition. The explicit separation of discrete and
continuous latent variables further prevents categorical variation from being
conflated with geometric transformations. Although the present framework is developed and evaluated in the setting of Lie Group VAEs, the central message extends beyond this specific model family. The central issue we study is
the mismatch between latent transformation structure and uncertainty behaviour
in the presence of non-commutative generative factors. This issue is not unique
to a single model family, but is relevant more generally to learned latent
representations in modern generative modelling, including representation spaces
used in latent generative pipelines. We focus on Lie Group VAEs because they
provide an explicit and mathematically tractable setting in which group
actions, commutators, and reconstruction-level effects can be analysed
rigorously. In this sense, Lie Group VAEs serve as a principled test-bed for
studying how algebraic structure should be reflected in uncertainty
quantification.  While our technical development is specialized to this setting, the broader lesson is that uncertainty estimates in structured latent models
should account not only for local geometry, but also for the algebraic
properties of the transformations acting on the latent space.  Our study has limitations. The pairwise finite-action diagnostics are local and
scale as \(O(d_c^2)\) in the continuous latent dimension \(d_c\), although this
can be mitigated by screening high-commutator pairs. They also do not capture
path-dependent composition effects. An important direction for our future work is therefore to extend the
analysis to trajectory-level descriptions of latent transformations, for example using signature-based
methods, and study analogous uncertainty principles in broader structured generative architectures, including diffusion models built on VAE backbones.






\bibliographystyle{plainnat}
\bibliography{ref}
\newpage

\appendix

\section{Technical appendices and supplementary material}

\subsection{Definitions of Variables }
\label{app:definitions}
Here, $x \in \mathcal{X}$ is the observed data (e.g., an image); $\hat z = E_{\mathrm{img}}(x)$ is the encoder feature before group action; $(\mu, \log\sigma^2) = E_{\mathrm{group}}(\hat z)$ are the parameters of the variational posterior; $t = \mu + \sigma \odot \varepsilon$ with $\varepsilon \sim \mathcal{N}(0,I)$ are Lie-algebra coordinates; $A(t)=\sum_j t_jA_j$ combines generator matrices $A_j$; $G(t)=\exp(A(t))$ is the corresponding group element; $z=\mathrm{vec}(G(t))$ is its vectorized group representation; $e_s=\mathrm{Emb}(z_{\mathrm{disc}})$ is the discrete embedding; $s'=G(t)e_s$ is the decoder input; and $D_{\mathrm{img}}(s')$ denotes the reconstruction.

\subsection{Proof of Theorem 1}
\label{app:theorom 1}
\begin{proof}
Consider a concatenated Jacobian block:
\[
J =
\left[
\frac{\partial x_{\mathrm{rec}}}{\partial t_i},
\frac{\partial x_{\mathrm{rec}}}{\partial t_j}
\right].
\]

Using the chain rule, each column can be expressed as:
\[
\frac{\partial x_{\mathrm{rec}}}{\partial t_k}
=
\frac{\partial D_{\mathrm{img}}}{\partial s'}
\frac{\partial s'}{\partial t_k},
\quad k \in \{i,j\}.
\]

Therefore, the Jacobian block factors as:
\[
J =
\frac{\partial D_{\mathrm{img}}}{\partial s'}
\left[
\frac{\partial s'}{\partial t_i},
\frac{\partial s'}{\partial t_j}
\right]
=
\frac{\partial D_{\mathrm{img}}}{\partial s'} Z,
\]
where
\[
Z =
\left[
\frac{\partial s'}{\partial t_i},
\frac{\partial s'}{\partial t_j}
\right]
\]
is the Jacobian of the Lie-transformed discrete embedding, with the discrete embedding held fixed.

The Gram matrix of \(J\) is:
\[
J^\intercal J
=
Z^\intercal
\left(\frac{\partial D_{\mathrm{img}}}{\partial s'}\right)^\intercal
\left(\frac{\partial D_{\mathrm{img}}}{\partial s'}\right)
Z
=
Z^\intercal M(s') Z,
\]
where
\[
M(s')
=
\left(\frac{\partial D_{\mathrm{img}}}{\partial s'}\right)^\intercal
\left(\frac{\partial D_{\mathrm{img}}}{\partial s'}\right)
\]
is the pullback metric induced by the decoder mapping, similar to the Riemannian geometric framework of \cite{Arvanitidis2017LatentSO}.
The singular values of \(J\) are the square roots of the eigenvalues of
\(J^\intercal J = Z^\intercal M(s') Z\). Therefore:
\[
U_{ij}^{\mathrm{manifold}}
=
\sigma_{\max}(J)
=
\sqrt{\lambda_{\max}(Z^\intercal M(s') Z)}.
\]

This measures the maximum Riemannian distortion of the linear map from the \((t_i,t_j)\) parameter space to the data space. Our approach uses the deterministic pullback metric, unlike Arvanitidis et al.\ (2017), which adds variance Jacobians for Gaussian decoders. 
When combined with the finite non-commutativity diagnostics, this sensitivity measure helps separate geometric decoder distortion from algebraic order effects.
\end{proof}

\subsection{Empirical calibration of $C$}
\label{app:Empirical calibration of C}

\paragraph{Phase-1: Empirical estimation.}
During an initial calibration phase (Phase~1), the model is trained without
enforcing the deformation--stability constraint. Instead, we periodically collect
diagnostic statistics for each unordered pair of latent dimensions $(i,j)$.
For each pair, we measure:
(i) an \emph{algebraic deviation} $D_{ij}$, which quantifies non-commutativity
of the corresponding Lie algebra generators, and
(ii) a \emph{decoder deviation} $\Delta_{ij}$, which measures the discrepancy
between reconstructions obtained by swapping the order of the associated
group actions in the decoder.

Let $\bar D_{ij}$ and $\bar \Delta_{ij}$ denote empirical averages of these
quantities accumulated over multiple diagnostic evaluations and used to
calibrate the scale of the constraint. From these statistics, we form a set of
scale ratios
\begin{equation}
r_{ij} \;=\; \frac{\bar \Delta_{ij}}{\bar D_{ij} + \epsilon},
\end{equation}
where $\epsilon>0$ is a small constant added for numerical stability.
Each ratio $r_{ij}$ provides an empirical estimate of how strongly a unit of
algebraic non-commutativity propagates to the observation space.

We then define the initial scale constant as a high percentile of this
distribution:
\begin{equation}
C_{\text{emp}} \;=\;
\operatorname{Percentile}_{p}\!\left(\{r_{ij}\}_{i<j}\right),
\end{equation}
with $p$ typically set to $90$. This percentile-based estimator yields a
robust calibration of $C$ that is insensitive to outliers while capturing the
typical scale of decoder deviations. The calibrated value is then used in
Phase~2 to set the scale of the deformation--stability constraint.

\paragraph{Phase-2: Constraint enforcement.}
In Phase-2, the calibrated constant $C$ is used to enforce consistency
between algebraic and decoder-level deviations via the hinge-squared penalty \ref{eq:pairwise-hinge-loss}.
The loss is averaged over all latent pairs and weighted by a coefficient
$\lambda_{\mathrm{unc}}$. This formulation penalizes only those pairs for which
algebraic non-commutativity would imply a larger decoder discrepancy than is
empirically observed.

\paragraph{Optional adaptive refinement.}
To account for potential changes in scale during training, we optionally
refine $C$ using a slow outer-loop update. We monitor the fraction of active
constraints,
\begin{equation}
f_{\text{active}}
=
\frac{1}{|\mathcal{P}|}
\sum_{(i,j)\in\mathcal{P}}
\mathbf{1}\!\left[C\,\bar D_{ij} > \bar \Delta_{ij}\right],
\end{equation}
and adjust $C$ multiplicatively to match a target activity level
$f_{\text{target}}$. Updates are applied after an initial freeze period and
clipped to a predefined interval $[C_{\min}, C_{\max}]$.

Overall, $C$ is treated as a data-dependent quantity rather than a fixed
hyperparameter: it is (i) empirically calibrated from diagnostic statistics,
(ii) interpretable as a conversion factor between algebraic and decoder-level
deviations, and (iii) optionally adapted to maintain a controlled level of
constraint activity throughout training.
Our objective is not to maximize reconstruction sensitivity, but to ensure that algebraic non-commutativity in the latent Lie algebra is faithfully expressed in reconstruction space. 

\subsection{Discrete latent losses}
\label{app:discrete_losses}
\paragraph{Mutual information regularization.}
To encourage semantic meaning in the discrete latent variable
$s \in \{0,1\}^K$, we introduce an auxiliary predictor
$Q_\psi$ that attempts to recover the discrete code from a reconstructed image.
Given a reconstruction $\hat x$, the predictor outputs logits
$q_\psi(s \mid \hat x)$.

We maximize a variational lower bound on the mutual information
between the discrete latent and the reconstruction, yielding the loss
\begin{equation}
\mathcal{L}_{\mathrm{MI}}
=
-\mathbb{E}_{x \sim p_{\mathrm{data}}}
\left[
\sum_{s} q(s \mid x)\log q_\psi(s \mid \hat x)
\right],
\end{equation}
where $q(s \mid x)$ is the discrete encoder distribution (implemented via a
Gumbel--Softmax relaxation).
This loss encourages the discrete code to be preserved through the
decoder and prevents it from being ignored. 

\subsection{Dataset details}
\label{app:dataset_details}

We evaluate on dSprites, 3DShapes, 3DCars, and CelebA. dSprites consists of
\(64 \times 64\) binary images generated from five ground-truth factors:
shape, scale, orientation, and \(x\)-/\(y\)-position. 3DShapes consists of
\(64 \times 64\) RGB images generated from six factors: object shape, scale,
orientation, object colour, wall colour, and floor colour. 3DCars consists of
17,568 RGB images of size \(64 \times 64 \times 3\), generated from three
ground-truth factors: elevation with 4 values, azimuth direction with 24 values,
and car model with 183 values. CelebA consists of 202,599 face images. For
CelebA, we crop the center \(128 \times 128\) region and resize it to
\(64 \times 64\) before training and evaluation.

\subsection{Experimental Settings}
\label{app:exp_setting}
For Beta-VAE, CFASL, and CLG-VAE, we use the training protocols and hyperparameter settings from the original papers and official implementations. For CLG-VAE, we follow the official implementation released in \texttt{CommutativeLieGroupVAE-Pytorch}. For CFASL, we likewise use the hyperparameter configuration provided in its official implementation for Commutative Lie Group VAE model (CFASL). We therefore do not repeat the full baseline hyperparameter settings here. 

For our model on the synthetic datasets (dSprites, 3DShapes, and Cars3D), we
retain the core Lie-related parameterization of CLG-VAE and report only the
main settings that differ or are central to our method. We use Adam with
learning rate $2\times 10^{-4}$ and latent dimension $10$. For dSprites and
Cars3D, we use Lie subgroup size $100$, while for 3DShapes we use subgroup size
$400$; the Lie subspace size is $10$ in all three cases, and the Lie algebra
initialization scale is $2\times 10^{-4}$. The phase-aware calibration uses
$C$ percentile $90$, $C$ update learning rate $0.05$, and diagnostic interval
$200$. For the main Phase~1$\rightarrow$Phase~2 runs, we use freeze-$C$ epochs
$10$. We set $C_{\max}=10000$ for dSprites and 3DShapes, and $C_{\max}=200$
for Cars3D. The remaining phase-specific weights are reported in
Table~\ref{tab:phase2_hparams}.
\begin{table}[H]
\centering
\footnotesize
\setlength{\tabcolsep}{4pt}
\caption{
Calibration and key training settings for the representative runs used in the
phase-aware diagnostics.
}
\label{tab:phase2_hparams}
\begin{tabular}{lccc}
\toprule
Hyperparameter & dSprites & 3DShapes & Cars3D \\
\midrule
$C_{\mathrm{emp}}$ & 289.87 & 207.47 & 0.688 \\
Phase 1 epochs & 41 & 41 & 251 \\
Phase 2 epochs & 59 & 89 & 69 \\
$C$ percentile $p_C$ & 90 & 90 & 90 \\
$C$ update lr $\eta_C$ & 0.05 & 0.05 & 0.05 \\
$C_{\min}$ & $10^{-4}$ & $10^{-4}$ & $10^{-4}$ \\
$C_{\max}$ & 10000 & 10000 & 200 \\
Freeze-$C$ epochs & 10 & 10 & 10 \\
Diagnostic interval $K_{\mathrm{diag}}$ & 200 & 200 & 200 \\
Batch size & 256 & 256 & 64 \\
$\lambda_{\mathrm{unc}}$ & 0.0008 & 0.0007 & 0.0001 \\
$\lambda_{\mathrm{MI}}$ & 0.6 & 0.6 & 0.5 \\
$\lambda_{\mathrm{shape}}$ & 0.1 & 0.1 & 0.2 \\
$\lambda_{\mathrm{rec\_shape}}$ & 0.02 & 0.02 & 0.02 \\
$\lambda_{\mathrm{shape\_sep}}$ & 0.001 & 0.001 & 0.001 \\
Gumbel-Softmax temperature $\tau$ & 0.67 & 0.67 & 1.5 \\
Shape embedding dim & 64 & 64 & 32 \\
Number of categories& 3 & 4 & 183 \\
\bottomrule
\end{tabular}

\vspace{2pt}
\parbox{\linewidth}{\footnotesize
\textit{Note.} $C_{\mathrm{emp}}$ denotes the final calibrated value at the end
of Phase~2 training.
}
\end{table}

For CelebA, we use a configuration with Lie subgroup size $400$, Lie subspace size $20$, and Lie algebra initialization scale $5\times10^{-4}$. The discrete latent is implemented as three Gumbel--Softmax categorical variables with category sizes $(3,4,4)$. Additional settings include Gumbel temperature $1.5$, $\lambda_{\mathrm{shape}}=0.05$, $\lambda_{\mathrm{MI}}=0.1$, and $\lambda_{\mathrm{unc}}=10$. The phase-aware calibration uses $C$ percentile $90$, $C$ update learning rate $0.02$, diagnostic interval $200$, freeze-$C$ epochs $5$, and $C_{\max}=200$.

\paragraph{Compute resources.}
All experiments were run on the Berzelius compute cluster using a single NVIDIA A100-SXM4-40GB GPU per run; no multi-GPU training was used. For the full Phase~1+Phase~2 training pipeline, approximate wall-clock times per run were \(17.33 \pm 0.13\) hours for dSprites, \(17.05 \pm 0.18\) hours for 3DShapes, \(1.65 \pm 0.02\) hours for Cars3D, and \(18.56 \pm 0.06\) hours for CelebA (mean \(\pm\) sample standard deviation over five seeds). Reported quantitative results were obtained from five independent random seeds per dataset/model configuration, giving total compute times of \(86{:}39{:}23\), \(85{:}13{:}44\), \(08{:}15{:}22\), and \(92{:}48{:}24\), respectively.

\subsection{Sensitivity analysis}
\label{app:sensitivity}
Sensitivity is evaluated relative to the baseline configuration using reconstruction loss and FVM. As can be seen from Table \ref{tab:sensitivity_3dshapes}, the model is not highly sensitive to reasonable changes in warmup length, constraint percentile, diagnostic frequency, or MI weight, with performance remaining broadly stable across settings.










\begin{table}[H]
\centering
\footnotesize
\caption{
Sensitivity analysis on 3DShapes. Each variant changes one hyperparameter
relative to the baseline configuration. Results are reported as mean
\(\pm\) standard deviation over five seeds.
}
\label{tab:sensitivity_3dshapes}
\begin{tabular}{lccc}
\toprule
\textbf{Variant}
& \textbf{Changed parameter}
& \textbf{Recon. loss} \(\downarrow\)
& \textbf{FVM} \(\uparrow\) \\
\midrule
Baseline
& --
& $15.67 \pm 1.87$
& $0.83856 \pm 0.03611$ \\

Warmup-20
& \(E_{\mathrm{warm}}=20\)
& $15.77 \pm 1.47$& $0.83856 \pm 0.03611$ \\

Warmup-60
& \(E_{\mathrm{warm}}=60\)
& $15.16 \pm 1.72$
& $0.84210 \pm 0.04068$ \\

\(C\)-perc.-70
& \(p_C=70\)
& $18.24 \pm 1.95$
& $0.83524 \pm 0.04247$ \\

\(C\)-perc.-95
& \(p_C=95\)
& $16.77 \pm 1.19$& $0.83856 \pm 0.03611$ \\

Diag.-100
& \(K_{\mathrm{diag}}=100\)
& $16.88 \pm 6.47$
& $0.82620 \pm 0.02489$ \\

Diag.-500
& \(K_{\mathrm{diag}}=500\)
& $17.98 \pm 2.21$
& $0.80774 \pm 0.01379$ \\

\(\lambda_{\mathrm{MI}}=0.25\)
& \(\lambda_{\mathrm{MI}}=0.25\)
& $17.13 \pm 3.54$
& $0.84695 \pm 0.03419$ \\

\(\lambda_{\mathrm{MI}}=0.5\)
& \(\lambda_{\mathrm{MI}}=0.5\)
& $15.91 \pm 2.07$
& $0.83805 \pm 0.04037$ \\

\(\lambda_{\mathrm{MI}}=1.0\)
& \(\lambda_{\mathrm{MI}}=1.0\)
& $16.41 \pm 2.93$
& $0.83090 \pm 0.01664$ \\

\bottomrule
\end{tabular}
\end{table}

\subsection{Additional phase-aware diagnostics}
\label{app:phase_diagnostics_additional}

Figure~\ref{fig:dsprites_phase_diag_appendix_full} shows the corresponding phase-aware diagnostic for dSprites. As in Cars3D, the Phase~2 deformation--stability ratio remains clearly above the reference boundary $\bar{R}=1$, indicating that decoder-level order sensitivity exceeds the calibrated lower bound implied by latent non-commutativity, while reconstruction loss remains stable.

\begin{figure}[t]
    \centering
    \textbf{\small dSprites}\par
    \vspace{0.15em}
    \includegraphics[width=0.70\columnwidth,trim=8 8 8 8,clip]{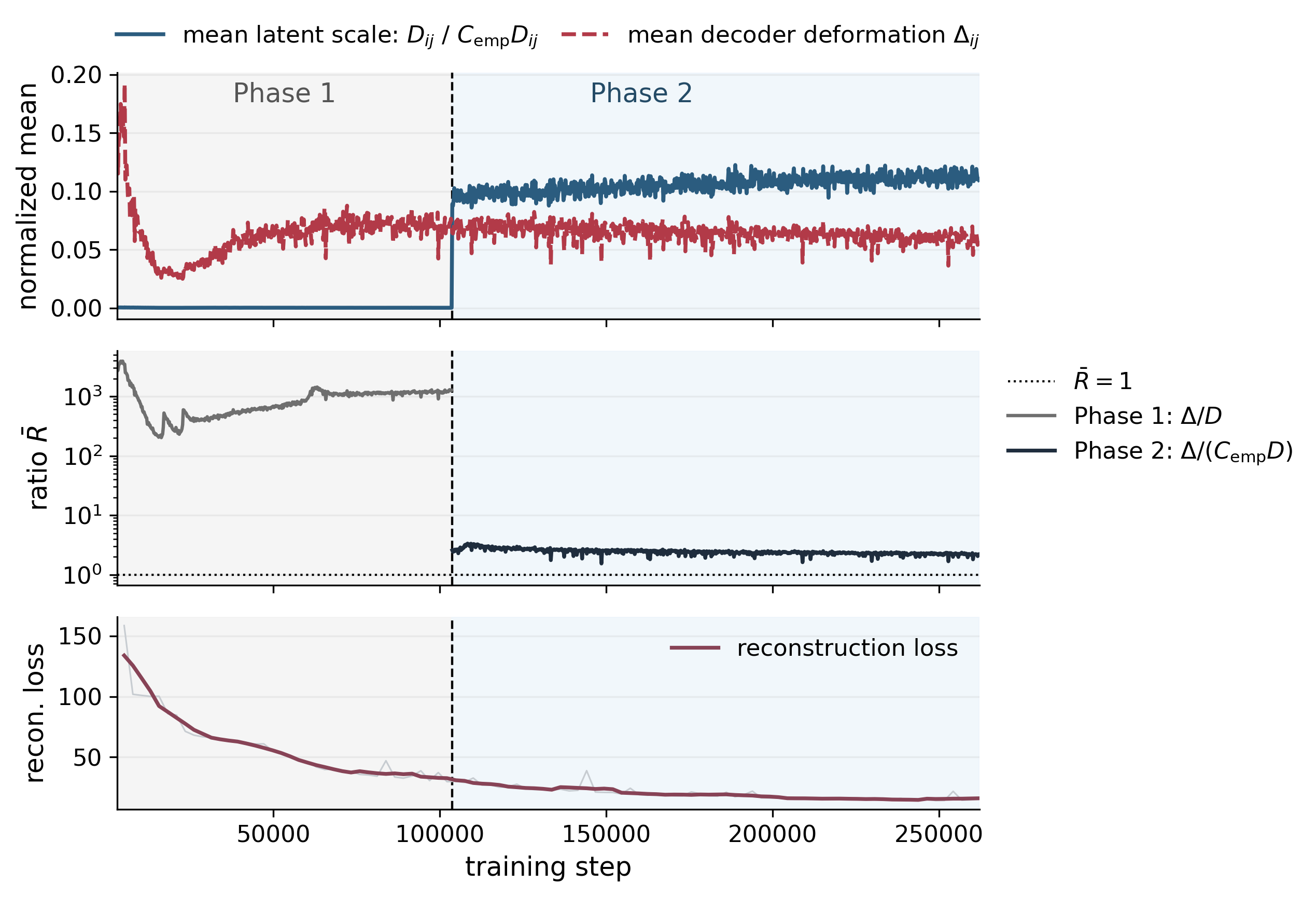}
    \caption{
    Additional phase-aware diagnostic for dSprites.
    Top: $\Delta_{ij}$ and a consistently scaled latent non-commutativity diagnostic across both phases.
    Middle: deformation--stability ratio $\bar{R}$ with reference boundary $\bar{R}=1$.
    Bottom: reconstruction loss.
    }
    \label{fig:dsprites_phase_diag_appendix_full}
\end{figure}






\subsection{CelebA: further qualitative analysis}
\label{app:celeba_further_qualitative}
\paragraph{Order-dependent latent interactions.}
Figure~\ref{fig:celeba_order_reversal} illustrates the non-commutative interaction between two learned latent directions in our model: $z_9$, corresponding to azimuth/pose, and $z_{14}$, corresponding to appearance/illumination. Each triplet applies the same two latent transformations with different step magnitudes but in reversed order. The first two images compare $z_9 \rightarrow z_{14}$ and $z_{14} \rightarrow z_9$, while the third highlights their comparison: yellow indicates overlap, whereas cyan and red indicate order-dependent differences. The red regions emphasize more localized changes around the eyebrows, eyes, nose, and jawline, while the cyan regions reflect broader shifts in the overall face contour and hair-face boundary. These examples suggest that reversing the order of the two latent actions can lead to visibly different decoded outputs.

\begin{figure}[H]
    \centering
    \includegraphics[width=\linewidth]{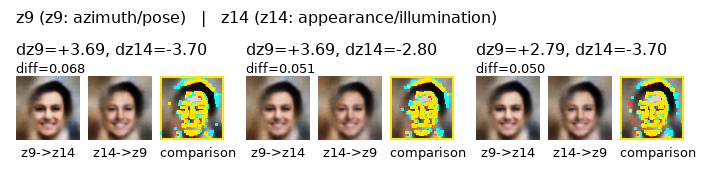}
    \caption{Order-dependent interaction between two learned latent directions.}
    \label{fig:celeba_order_reversal}
\end{figure}



\end{document}